\documentclass[10pt,twocolumn,letterpaper]{article}

\usepackage{iccv}
\usepackage{times}
\usepackage{epsfig}
\usepackage{graphicx}
\usepackage{amsmath}
\usepackage{amssymb}

\usepackage{tabularx}
\usepackage{adjustbox}
\usepackage{array}
\usepackage{booktabs}
\usepackage{colortbl}
\usepackage{makecell}
\usepackage{float,wrapfig}
\usepackage{hhline}
\usepackage{multirow}
\usepackage{animate} 
\usepackage{subfig}
\usepackage{lipsum}
\usepackage{bbding}
\usepackage{amsthm} 

\usepackage{cite}
\usepackage[american]{babel}

\usepackage[pagebackref=true,breaklinks=true,letterpaper=true,colorlinks,bookmarks=false]{hyperref}

\iccvfinalcopy 

\def\iccvPaperID{6321} 

\ificcvfinal\fi

\begin{document}
	
	
	\title{Orthogonal Jacobian Regularization for Unsupervised Disentanglement in Image Generation}
	
	\author{\small{Yuxiang Wei$^{1}$\footnotemark[1], \ \ Yupeng Shi$^{1}$, Xiao Liu$^{2}$, Zhilong Ji$^{2}$, Yuan Gao$^{2}$, Zhongqin Wu$^{2}$, Wangmeng Zuo$^{1, 3}$ $^{(}$\Envelope$^)$ } \\
		$^1$\small{Harbin Institute of Technology}, \ \ $^2$\small{Tomorrow Advancing Life}, \ \ $^3$\small{Pazhou Lab, Guangzhou}\\
		{\tt\small{ \{yuxiang.wei.cs, csypshi\}@gmail.com} } {\tt {\small\{liuxiao15, jizhilong, gaoyuan23, wuzhongqin\}@tal.com} } \\ {\tt{\small{wmzuo@hit.edu.cn}}} \\    \\
	}

	\maketitle
	
	\renewcommand{\thefootnote}{\fnsymbol{footnote}}
	\footnotetext[1]{This work was done when Yuxiang Wei was a research intern at TAL}
	
	\begin{abstract}
		\vspace{-1em}
		Unsupervised disentanglement learning is a crucial issue for understanding and exploiting deep generative models. Recently, SeFa tries to find latent disentangled directions by performing SVD on the first projection of a pre-trained GAN. However, it is only applied to the first layer and works in a post-processing way. Hessian Penalty minimizes the off-diagonal entries of the output's Hessian matrix to facilitate disentanglement, and can be applied to multi-layers.
		However, it constrains each entry of output independently, making it not sufficient in disentangling the latent directions (e.g., shape, size, rotation, etc.) of spatially correlated variations. 
		In this paper, we propose a simple \textbf{Or}th\textbf{o}gonal \textbf{Ja}cobian \textbf{R}egularization (\textbf{OroJaR}) to encourage deep generative model to learn disentangled representations. It simply encourages the variation of output caused by perturbations on different latent dimensions to be orthogonal, and the Jacobian with respect to the input is calculated to represent this variation. 
		We show that our OroJaR also encourages the output's Hessian matrix to be diagonal in an indirect manner. In contrast to the Hessian Penalty, our OroJaR constrains the output in a holistic way, making it very effective in disentangling latent dimensions corresponding to spatially correlated variations. 
		Quantitative and qualitative experimental results show that our method is effective in disentangled and controllable image generation, and performs favorably against the state-of-the-art methods. Our code is available at \url{https://github.com/csyxwei/OroJaR}.

	\end{abstract}
	
	\vspace{-1em}
	\section{Introduction}
	
	\begin{figure}
		\centering
		\includegraphics[width=1\linewidth]{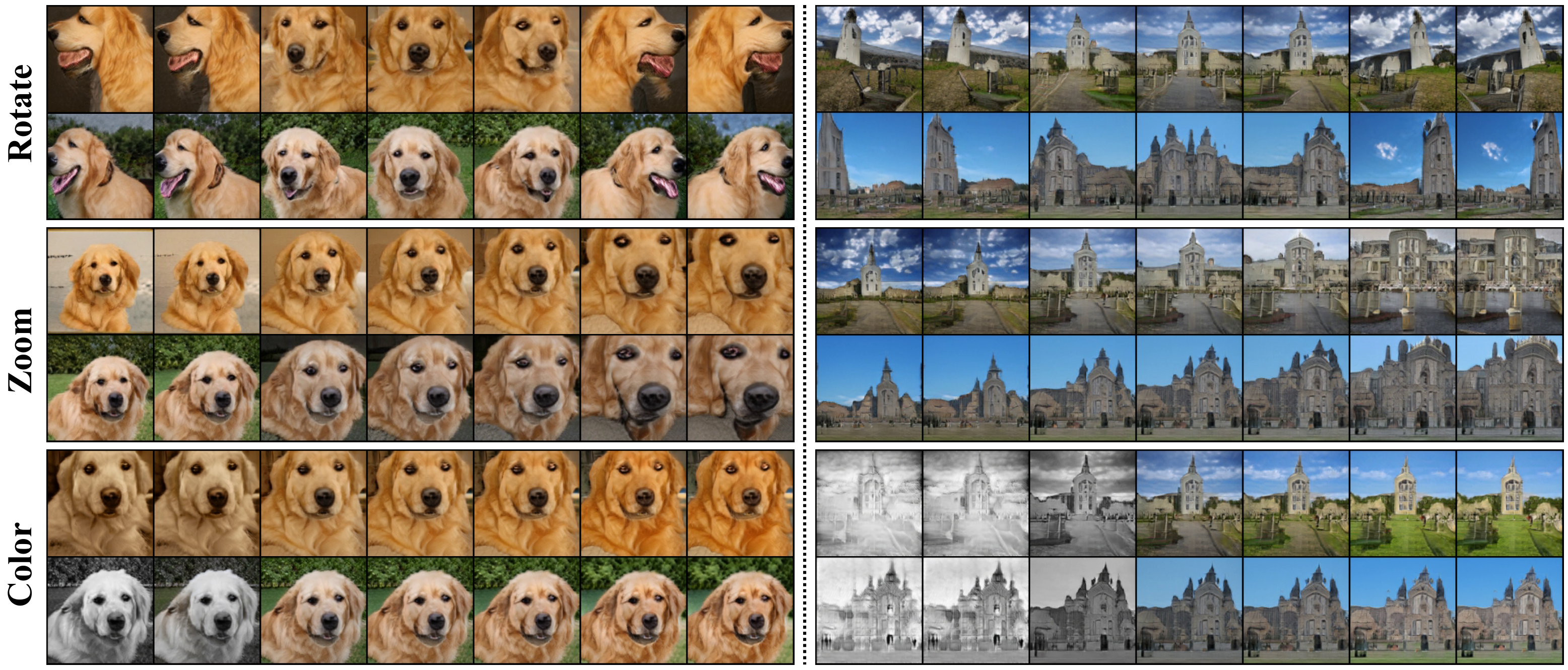}
		\caption{Examples of orthonormal directions learned by our method in BigGAN conditioned to synthesize ImageNet Golden Retrievers or Churches. Moving across a row, we move a latent code along a single linear direction in $\mathbf{z}$-space.}
		\label{fig:biggan}
		\vspace{-2em}
	\end{figure}
	
	In a disentangled representation, each dimension corresponds to the change in one factor of variation (FOV), while being independent to changes in other factors~\cite{bengio2013deep}. Learning disentangled representations from a given dataset is a major challenge in artificial intelligence, and can be beneficial to many computer vision tasks, such as domain adaptation \cite{yang2019unsupervised, peng2019domain}, controllable image generation \cite{shen2020closed, voynov2020unsupervised, zhu2020learning, peebles2020hessian}, and image manipulation \cite{shen2020interfacegan}.
	
	In the recent few years, unsupervised disentanglement learning has attracted intensive attention, owing to its importance in understanding generative models \cite{peebles2020hessian, shen2020closed} and extensive applications in various vision tasks \cite{yang2019unsupervised, shen2020interfacegan}. 
	Based on two representative generative models, \ie Variational Autoencoder (VAE) \cite{kingma2013auto} and Generative Adversarial Networks (GAN)\cite{goodfellow2014generative}, many disentanglement methods \cite{higgins2016beta, kim2018disentangling, jeong2019learning, chen2018isolating, dupont2018learning, chen2016infogan, zhu2020learning, peebles2020hessian, shen2020closed} have been proposed. 
	VAE-based methods, such as $ \beta $-VAE \cite{higgins2016beta}, FactorVAE \cite{kim2018disentangling}, $ \beta $-TCVAE \cite{chen2018isolating}, \etc, attain disentanglement mainly by enforcing the independence in the latent variables. 
	However, their disentanglement performance and the visual quality of generated images remain quite limited. With the progress in Generative Adversarial Networks (GAN) \cite{goodfellow2014generative}, many GAN-based disentanglement methods have been proposed \cite{chen2016infogan, zhu2020learning, peebles2020hessian, shen2020closed}. SeFa \cite{shen2020closed} learns the disentangled latent directions by directly decomposing the weight of the first fully-connected layer of a pre-trained GAN. However, it is only applied to the first layer of the generator model and works in a post-processing way, which limits the performance of disentanglement. Hessian Penalty \cite{peebles2020hessian} encourages to learn a disentangled representation by minimizing the off-diagonal entries of the output's Hessian matrix with respect to its input. However, it uses a max function to extend the regularization from scalar-valued functions to vector-valued functions, yet treats each entry of the output independently, making it not sufficient in disentangling the latent directions (\eg, shape, size, rotation, \etc) corresponding to spatially correlated variations.
	
	Inspired by Hessian Penalty~\cite{peebles2020hessian} and SeFa~\cite{shen2020closed}, we propose a simple regularization term to encourage the generative model to learn disentangled representations. Our method is based on a straightforward intuition: when perturbing a single dimension of the network input, we would like the change in the output to be independent (and also uncorrelated) with those caused by the other input dimensions. To this end, the output's Jacobian matrix is calculated to represent the change caused by the latent input. To encourage the changes caused by different latent dimensions to be uncorrelated, we simply constrain the Jacobian vector of each dimension to be orthogonal. In contrast to Hessian Penalty, we constrain the change in a holistic way, thereby making it very competitive in disentangling latent dimensions corresponding to spatially correlated variations.
	We call this regularization term as \textbf{Or}th\textbf{o}gonal \textbf{Ja}cobian \textbf{R}egularization (\textbf{OroJaR}). In Sec.~\ref{sec:relation}, we show that our OroJaR also constrains the Hessian matrix to be diagonal in an indirect way. On the other hand, our OroJaR can be treated as an end-to-end generalization of SeFa on multiple layers, which is also beneficial to the disentanglement performance. In practice, due to the fact that computing the Jacobian matrices during training is time consuming, we approximate it via a first-order finite difference approximation to accelerate training.
	
	Experiments show that our OroJaR performs favorably against the state-of-the-art methods~\cite{shen2020closed,peebles2020hessian,zhu2020learning} for unsupervised disentanglement learning on three datasets (\ie,  Edges+Shoes \cite{yu2014fine}, CLEVR \cite{peebles2020hessian}, and Dsrpites \cite{dsprites17}). 
	Moreover, our OroJaR can be used to explore directions of meaningful variation in the latent space of pre-trained generators. From Fig.~\ref{fig:biggan}, our method is effective in finding the disentangled latent directions (\eg, rotation, zoom and color, \etc) in BigGAN pre-trained on ImageNet.
	
	The contributions of this work can be summarized as:
	\vspace{-0.5em}
	\begin{itemize}
		\setlength{\itemsep}{0pt}
		\setlength{\parsep}{0pt}
		\setlength{\parskip}{0pt}
		\item We present a simple Orthogonal Jacobian Regularization (OroJaR)  to encourage the deep generative model to learn better disentangled representations.
		\item OroJaR can be applied to multiple layers of the generator, constrains the output in a holistic way, and indirectly encourages the Hessian matrix to be diagonal.
		\item Extensive experiments show the effectiveness of
		our proposed method in learning and exploring disentangled representations, especially those corresponding to spatially correlated variations.
	\end{itemize}
	
	\section{Related Work}
	
	\subsection{Disentanglement Learning in VAE}
	
	Variational Autoencoder (VAE) \cite{kingma2013auto} has been widely adopted in state-of-the-art disentanglement methods \cite{locatello2019challenging, higgins2016beta, kim2018disentangling, chen2018isolating, kumar2017variational, ding2020guided, karaletsos2015bayesian, jha2018disentangling}. $ \beta $-VAE \cite{higgins2016beta} introduced an adjustable hyperparameter $ \beta  > 1$ on the KL divergence between the variational posterior and the
	prior to VAE for benefiting disentangled representations, but meanwhile, it sacrificed the reconstruction result.
	Based on $ \beta $-VAE, \cite{kim2018disentangling} and \cite{chen2018isolating} introduced the total correlation (TC) term in order to improve disentanglement performance.
	DIP-VAE \cite{kumar2017variational} used moment matching to penalize the divergence between aggregated posterior and the prior to encourage the disentanglement.
	Guided-VAE \cite{ding2020guided} used an additional discriminator to guide the unsupervised disentanglement learning and learned the latent geometric transformation and principal components.
	Additionally, JointVAE \cite{dupont2018learning} and CascadeVAE \cite{jeong2019learning} tried to simultaneously learn disentangled continuous and discrete representations in an unsupervised manner.
	To sum up, most existing VAE-based methods disentangle the variations mainly by factorizing aggregated posterior, but usually suffer from low-quality image generation ability.
	
	\subsection{Disentanglement Learning in GAN}
	
	Two kinds of methods, \ie, two-stage and one-stage ones, have been mainly investigated for finding disentangled representations in GAN \cite{goodfellow2014generative}.
	The two-stage methods identify disentangled and interpretable directions in the latent space of a pre-trained GAN.
	While the one-stage methods encourage disentanglement during GAN training by introducing appropriate extra regularization.
	
	\noindent\textbf{Interpretable directions in the latent space.}
	Several unsupervised methods have been suggested for discovering interpretable directions in the latent space of a pre-trained GAN~\cite{bau2018gan,harkonen2020ganspace,shen2020closed, voynov2020unsupervised, shen2020interfacegan}.
	Voynov \etal \cite{voynov2020unsupervised} searched a set of directions that can be easily distinguished from each other by jointly learning a candidate matrix and a classifier such that the semantic directions in the matrix can be properly recognized by the classifier. 
	H{\"a}rk{\"o}nen \etal \cite{harkonen2020ganspace} performed PCA on the sampled data to find the important and meaningful directions in the style space of StyleGAN.
	Shen \etal \cite{shen2020closed} searched the interpretable directions by performing SVD on the weight of the first layer of a pre-trained GAN.
	Wang \etal \cite{wang2021geometry} unified these approaches by treating them as special cases of computing the spectrum of the Hessian for the LPIPS model \cite{zhang2018unreasonable} with respect to the input.
	%
	%
	Nonetheless, two-stage methods only work in a post-processing manner for pre-trained GANs, and generally fail to discover the disentangled components that are nonlinear in the latent space.
	
	\noindent\textbf{Disentanglement learning with regularization.}
	Instead of post-processing, studies have also been given to achieve disentanglement by incorporating extra regularization~\cite{chen2016infogan, peebles2020hessian, zhu2020learning, tran2017disentangled, donahue2017semantically, nguyen2019hologan, ramesh2018spectral} in GAN training.
	%
	%
	InfoGAN \cite{chen2016infogan} learned the disentangled representations by maximizing the mutual information between the input latent variables and the output of the generator.
	Zhu \etal \cite{zhu2020learning} presented a variation predictability loss that encourages disentanglement by maximizing the mutual information between latent variations and corresponding image pairs.
	Peebles \etal \cite{peebles2020hessian} proposed the Hessian Penalty to make the generator have diagonal Hessian with respect to the input. 
	However, the max operator is used to extend Hessian Penalty for handling vector-valued output.
	As a result, it constrains each entry of output independently and is not sufficient in disentangling the latent directions corresponding to spatially correlated variations.
	Our OroJaR is motivated by Hessian Penalty \cite{peebles2020hessian} and SeFa \cite{shen2020closed}.
	It can be treated as an end-to-end generalization of SeFa to {multiple layers}, and constrain the change caused by latent dimension in a holistic way.  
	Experiments also show that OroJaR is more effective in disentangling latent dimensions corresponding to spatially correlated variations.
	
	\subsection{Orthogonal Regularization} \label{sec:ortho}
	Many recent studies have been given to incorporate the orthogonality for improving deep network training~\cite{wang2015deep, jia2017improving, odena2018generator, brock2016neural, rodriguez2016regularizing, wang2020orthogonal}.
	Wang \etal~\cite{wang2015deep} imposed orthogonal regularization on the weighting parameters with the form $ \Vert \mathbf{W}^T\mathbf{W} - \mathbf{I} \Vert_2 $, where $ \mathbf{W} $ is the weight matrix and $\mathbf{I}$ is an identity matrix.
	Jia \etal~\cite{jia2017improving} encouraged the orthogonality by bounding the singular values of the weight matrix in a narrow range around 1.
	For improving image generation quality, BigGAN~\cite{brock2018large} introduced a ``truncation trick'' by removing the diagonal terms from the regularization.
	Bansal \etal \cite{bansal2018can} introduced another orthogonal regularization by considering both $ \Vert \mathbf{W}^T\mathbf{W} - \mathbf{I} \Vert_2 $ and $ \Vert \mathbf{W}\mathbf{W}^T - \mathbf{I} \Vert_2 $.

	Besides the weight matrix, orthogonal regularization can also be used to constrain the latent space and Jacobian matrix.
	PrOSe~\cite{shukla2019product} parameterized the latent space representation as a product of orthogonal spheres to learn disentangled representations. 
	Odena \etal \cite{odena2018generator} introduced a regularization term to encourage the singular values of Jacobian matrix $\mathbf{J}$ of the generator to lie within a range. 
	It can also constrain $\mathbf{J}$ to be orthonormal to a scale when the range is sufficiently narrow.
	StyleGAN2 \cite{karras2020analyzing} presented a path length regularization which implicitly encourages the Jacobian matrix of the generator to be orthonormal up to a global scale.
	While the regularizers in~\cite{odena2018generator,karras2020analyzing} are adopted to improve the quality of the learned generator, our OroJaR is introduced to encourage the generator to learn disentangled representations. 
	Moreover, \cite{odena2018generator,karras2020analyzing} encourage the Jacobian vectors to be \emph{orthonormal} to a global scale, while our OroJaR only constrains them to be \emph{orthogonal}.
	
	\section{Proposed Method}

	In this section, we first describe the proposed Orthogonal Jacobian Regularization (OroJaR) for learning disentangled representations.
	Then, a first-order finite difference approximation is introduced to accelerate training.
	Finally, we discuss its connections with the related disentanglement methods, \ie, SeFa~\cite{shen2020closed} and Hessian Penalty~\cite{peebles2020hessian}.
	
	\subsection{Orthogonal Jacobian Regularization}
	Suppose $ G $: $ \mathbf{x} = G(\mathbf{z}) $ is a deep generative model.
	Here, $\mathbf{z} = [z_1, ..., z_i, ..., z_m]^T \in \mathbb{R}^{m} $ denotes the input vector to $G$, and $z_i$ denotes the $i$-th latent dimension.
	$ \mathbf{x} \in \mathbb{R}^n$ denotes the output of $ G $, and $\mathbf{x}_d = G_d(\mathbf{z})$ is further introduced to denote the the $d$-th layer's output of $G$.
	In terms of disentangled representation, each latent dimension is assumed to control the change in one factor of variation.
	That is, the changes caused by two different latent dimensions $ z_i $ and $ z_j $ should be independent (and also uncorrelated).

	In our method, we use the Jacobian vector, \ie, $ \frac{\partial G_d}{\partial z_i} $, to represent the change caused by the perturbation on the latent dimension $ z_i $.
	Then, for encouraging disentangled representation, we constrain their Jacobian vectors of different latent dimensions to be orthogonal,
	\vspace{-0.5em}
	\begin{equation}
		\left[ \frac{\partial G_d}{\partial z_i}\right]^T \frac{\partial G_d}{\partial z_j} = 0 .
		\label{eq:l1}
		\vspace{-0.5em}
	\end{equation}
	It is worth noting that, the orthogonality of two vectors indicates that they are uncorrelated, which also encourages the changes caused by different latent dimensions to be independent.
	
	Taking all latent dimensions into account, we present the Orthogonal Jacobian Regularization (OroJaR) for helping deep generative model to learn disentangled representations,
	\vspace{-1em}
	\begin{equation}
		\mathcal{L}_J (G) \!=\! \sum_{d=1}^{D} \Vert \mathbf{J}_d^T \mathbf{J}_d \circ ( \mathbf{1} - \mathbf{I}) \Vert \!=\! \sum_{d=1}^{D} \sum_{i=1}^{m} \sum_{j \neq i} \left\vert \left[\frac{\partial G_d}{\partial z_i} \right]^T  \frac{\partial G_d}{\partial z_j} \right\vert^2,
		\label{eq:l2}
	\end{equation}
	where $ \mathbf{J}_d = [\mathbf{j}_{d,1}, ..., \mathbf{j}_{d,i,}, \mathbf{j}_{d,m}]$ denotes the Jacobian matrix of $ G_d $ with respect to $ \mathbf{z} $, and $\circ$ denotes the Hadamard product.
	$\mathbf{I}$ denotes an identity matrix, and $\mathbf{1}$ is a matrix of all ones.
	In particular, we use $\mathbf{j}_{d,i} = \frac{\partial G_d}{\partial z_i}$ to represent a Jacobian vector.

	Our OroJaR constrains the change of output caused by latent dimension in a holistic way.
	To illustrate this point, we let $\mathbf{j}_{d}^{ij} = \mathbf{j}_{d,i} \circ \mathbf{j}_{d,j}$.
	Then, $\mathbf{j}_{d,i}^T \mathbf{j}_{d,j}$ can be equivalently obtained as the sum of all the elements of $\mathbf{j}_{d}^{ij}$.
	Obviously, OroJaR only constrains the summation of $\mathbf{j}_d^{ij}$ is small, and each element of $\mathbf{j}_{d}^{ij}$ can be positive/negative as well as large/small.
	Thus, our OroJaR does not impose any individual constraint on the elements of $\mathbf{j}_{d}^{ij}$.
	We note that the changes caused by many latent semantic factors (\eg, shape, size, rotation, \etc) usually are spatially correlated, and are better to be constrained in a holistic manner.
	In comparison, Hessian Penalty~\cite{peebles2020hessian} uses a max function for aggregating the Hessian matrix of vector-valued output.
	It actually requires the off-diagonal entries of the Hessian matrix to be small for each element of the output, thereby making it not sufficient in disentangling the factors of complex and spatially correlated variations.

	\subsection{Approximation for Accelerated Training}
	During training, it is time consuming to compute the Jacobian matrices in Eqn.~(\ref{eq:l2}) when $ m $ is large. Following \cite{peebles2020hessian, hutchinson1989stochastic}, we use the Hutchinson's estimator to rewrite Eqn.~(\ref{eq:l2}) as:
	\vspace{-0.5em}
	\begin{equation}
		\small
		\mathcal{L}_J (G) \!=\! \sum_{d=1}^{D} \text{Var}_\mathbf{v} \left[ \mathbf{v}^T (\mathbf{j}_{d}^T \mathbf{j}_{d}) \mathbf{v} \right] \!=\! \sum_{d=1}^{D} \text{Var}_\mathbf{v} \left[(\mathbf{j}_{d} \mathbf{v})^T \mathbf{j}_{d} \mathbf{v} \right],
		\label{eq:hum}
		\vspace{-0.2em}
	\end{equation}
	where $ \mathbf{v} $ are Rademacher vectors (each entry has equal probability of being -1 or 1), and $\text{Var}_\mathbf{v}$ denotes the variance. $ \mathbf{j}_{d} \mathbf{v} $ is the first directional derivative of $ G $ in the direction $ \mathbf{v} $ times $ \vert \mathbf{v} \vert $. 
	$ \mathbf{j}_{d} \mathbf{v} $ can be efficiently computed by a first-order finite difference approximation \cite{richardson1954introduction}:
	\vspace{-0.5em}
	\begin{equation}
		\mathbf{j}_{d} \mathbf{v} = \frac{1}{\epsilon} [G(\mathbf{z} + \epsilon \mathbf{v}) - G(\mathbf{z})],
		\label{eq:l3}
		\vspace{-0.2em}
	\end{equation}
	where $ \epsilon \textgreater 0 $  is a hyperparameter that controls the granularity of the first directional derivative estimate. In our implementation, we use $\epsilon  = 0.1$.
	
	\subsection{Applications in Deep Generative Models \label{sec: app}}
	
	Our OroJaR can be applied to many generative models, and here we consider the representative Generative Adversarial Networks (GAN) \cite{goodfellow2014generative}. The OroJaR can be applied to GAN in two ways.

	\noindent\textbf{Training from scratch}. For GAN, the discriminator $ D $ and generator $ G $ are respectively trained using $\mathcal{L}_{D}$ and $\mathcal{L}_{G}$,
	\begin{equation}
		\mathcal{L}_{D} = \mathbb{E}_\mathbf{x} [f(D(\mathbf{x}))] + \mathbb{E}_\mathbf{z} [f(1 - D(G(\mathbf{z})))],
		\label{eq:dadv}
	\end{equation}
	\vspace{-1em}
	\begin{equation}
		\mathcal{L}_{G} = \mathbb{E}_\mathbf{z} [f(1 - D(G(\mathbf{z})))],
		\label{eq:gadv}
	\end{equation}
	where $ f $ is a model-specific mapping adopted by GAN.
	In order to apply OroJaR to GAN training, we simply modify the loss for the generator as,
	\begin{equation}
		\mathcal{L}_{G}^{oro} = \mathbb{E}_\mathbf{z} [f(1 - D(G(\mathbf{z})))] + \lambda   \mathbb{E}_\mathbf{z} [ \mathcal{L}_J (G(\mathbf{z})) ],
		\label{eq:dadvwj}
	\end{equation}
	where $ \lambda$ is a trade-off hyper-parameter.
	Incorporating $\mathcal{L}_J (G)$ into GAN training is beneficial to learning disentangled representation, and encourages $G$ to achieve controllable and disentangled image generation.
	
	\noindent
	\textbf{Apply to pre-trained generator}. Analogous to Hessian Penalty \cite{peebles2020hessian}, our OroJaR can be used to identify interpretable directions in latent space of a pretrained generator. Specifically, we introduce a learnable orthonormal matrix $ \mathbf{A} \in \mathbb{R}^{m \times N} $, where $ N $ denotes the number of orthonormal directions we want to learn and $ m $ is the latent dimension; the columns of $ \mathbf{A} $ store the directions we are learning. After appling the OroJaR to pre-trained $ G $, $ \mathbf{A} $ is optimized by:
	\begin{equation}
		\mathbf{A}^* = \arg \min_{\mathbf{A}} \mathbb{E}_{\mathbf{z}, \omega_i} \mathcal{L}_J (G(\mathbf{z} + \eta \mathbf{A} \omega_i)), 
		\label{eq:biggan}
	\end{equation}
	where $ \omega_i \in \{0, 1\}^N $ is a one-hot vector which indexes the columns of $ \mathbf{A} $ and $ \eta $ is a scalar which controls how far $\mathbf{z}$ should move in the direction. The difference with Eqn.~(\ref{eq:dadvwj}) is the OroJaR is now taken  \wrt $ \omega_i $ instead of $ \mathbf{z} $. In our training, we use $\eta  = 1$. After optimization,  $ \mathbf{A} $ can be used to edit the generated images by $G(\mathbf{z} + \eta \mathbf{A} \omega_i)$.
	
	\subsection{Connections with SeFa and Hessian Penalty}
	\label{sec:relation}
	
	We further discuss connections and differences of OroJaR with two representative disentanglement learning methods, \ie, SeFa~\cite{shen2020closed} and Hessian Penalty~\cite{peebles2020hessian}.
	
	\noindent\textbf{SeFa}. SeFa~\cite{shen2020closed} performs SVD on the weight matrix $\mathbf{W}\in \mathbb{R}^{m_1 \times m}$ of the first layer to discover semantically meaningful directions in the latent space of pre-trained GAN.
	Let $\mathbf{W} = \mathbf{U} \boldsymbol{\Lambda} \mathbf{V}^T$ be the singular value decomposition (SVD) of $\mathbf{W}$.
	In SeFa~\cite{shen2020closed}, the semantically meaningful directions are given as the column vectors of $\mathbf{V}$.
	We introduce $\mathbf{z}^{\prime} = \mathbf{V}^T \mathbf{z}$ and $\mathbf{W}^{\prime} = \mathbf{U} \boldsymbol{\Lambda}$, and define $G_1(\mathbf{z}) = \mathbf{W}\mathbf{z}$ and $G^{\prime}_1(\mathbf{z}^{\prime}) = \mathbf{W}^{\prime}\mathbf{z}^{\prime}$.
	One can easily see that (i) each dimension of $\mathbf{z}^{\prime}$ corresponds to a semantically meaningful direction discovered by SeFa~\cite{shen2020closed}.
	(ii) $G^{\prime}_1(\mathbf{z}^{\prime})$ is equivalent with $G_1(\mathbf{z})$, \ie, $G_1(\mathbf{z}) = G^{\prime}_1(\mathbf{z}^{\prime})$.
	(iii) Hard orthogonal Jacobian constraint can be attained, \ie,
	\begin{equation}
		\left[ \frac{\partial G^{\prime}_1}{\partial z^{\prime}_i}\right]^T \frac{\partial G^{\prime}_1}{\partial z^{\prime}_j} = 0.
	\end{equation}
	Thus, SeFa~\cite{shen2020closed} can be treated as a special case of our OroJaR by finding the globally optimum of $\mathcal{L}_J$ defined only on the first layer $G^{\prime}_1(\mathbf{z}^{\prime})$ and keeping the parameters of all other layers unchanged.
	In contrast to SeFa, our OroJaR can be deployed to multiple layers and be jointly optimized with GAN in an end-to-end manner, thereby being beneficial to learn better disentangled representation.
	
	\begin{figure*}
		\centering
		\includegraphics[width=0.99\linewidth]{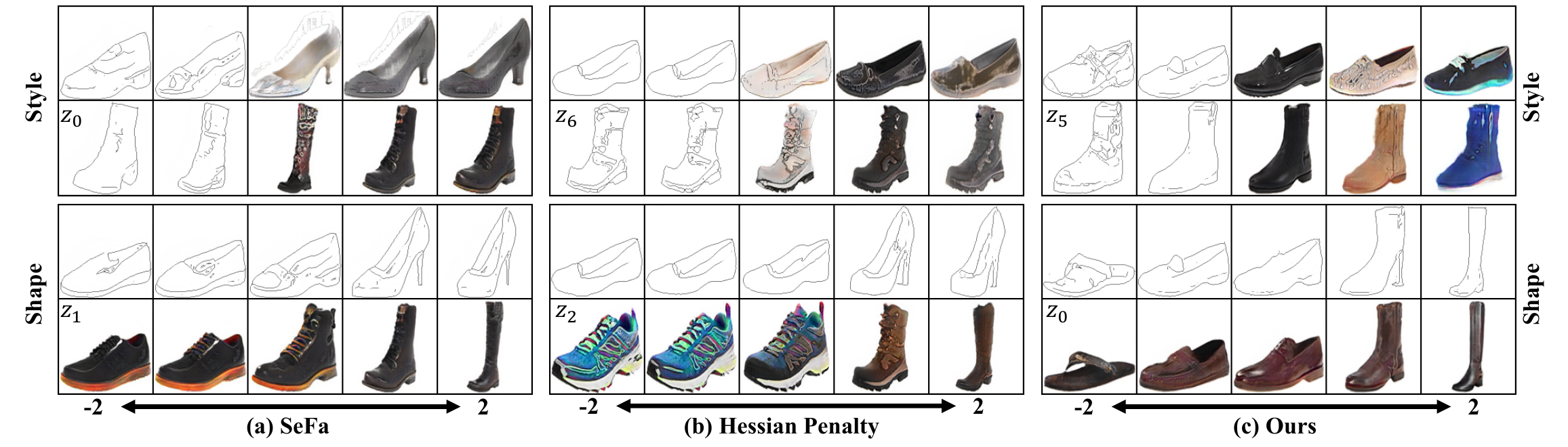}
		\vspace{-1em}
		\caption{
			Comparison of disentanglement quality by our OroJaR, Hessian Penalty \cite{peebles2020hessian} and SeFa \cite{shen2020closed} on Edges+Shoes. 	
			For each method, we randomly sample two 12-dimensional Gaussian vectors. 
			We select two interpretable dimensions to display, \ie, the shape and style of shoes, and every two rows correspond to one interpretable dimension. Moving across a row, we vary the value of dimension $ z_i $ from $-2$ to $+2$ while keeping the other 11 dimensions unchanged.
		}
		\label{fig:visual_shoes}
		\vspace{-1.3em}
	\end{figure*}
	
	\begin{figure*}
		\centering
		\includegraphics[width=0.99\linewidth]{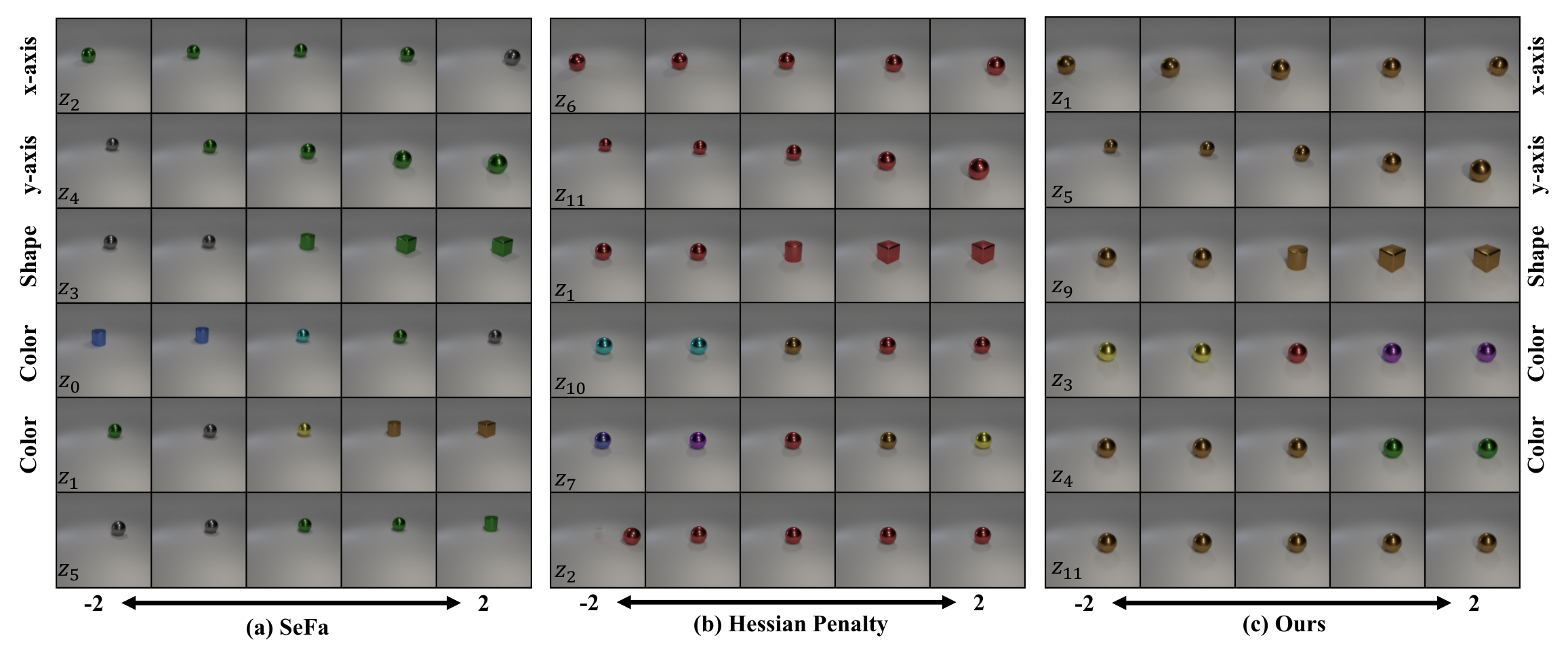}
		\vspace{-1.2em}
		\caption{Comparison of disentanglement quality by our OroJaR, Hessian Penalty \cite{peebles2020hessian} and SeFa \cite{shen2020closed} on CLEVR-Simple. 
			Our method has the ability to shrink the latent space when it is overparameterized.
			So we only show the top six activeness scoring dimensions (See Fig.~\ref{fig:activeness} and Sec.~\ref{sec:quantitative}). \textbf{(a)} SeFa disentangles the positions (top two rows). However it entangles the color with the shape variation (3rd-5th rows). \textbf{(b) } Hessian Penalty learns to control the vertical position, shape and color of the object independently (2nd-5th rows). However, horizontal position is unexceptedly controlled by two dimensions (1st and 6th rows). \textbf{(c)} Our method can successfully disentangle the four factors (two dimensions for color variation, but the colors controlled by them are non-overlapping) in CLEVR-Simple, and achieves better disentanglement performance.}
		\label{fig:visual_simple}
		\vspace{-1.3em}
	\end{figure*}
	
	\noindent\textbf{Hessian Penalty}.~To learn disentangled representation, Hessian Penalty~\cite{peebles2020hessian} encourages the generator to have diagonal Hessian of the output with respect to the input.
	By only considering two latent dimensions $ z_i $ and $ z_j $, the objective of Hessian Penalty can be written as,
	\begin{equation}
		\left \Vert \frac{\partial^2 G}{\partial z_i \partial z_j} \right \Vert^2 = 0.
		\label{eq:hessian}
	\end{equation}
	The left term can be further decomposed into 4 components,
	\vspace{-1em}
	\begin{equation}
		\small
		\begin{split}
			&\left \Vert \frac{\partial^2 G}{\partial z_i \partial z_j} \right \Vert^2 = \left[ \frac{\partial^2 G}{\partial z_i \partial z_j} \right]^T \frac{\partial^2 G}{\partial z_j \partial z_i} \\
			\approx &  \frac{1}{\delta z_i \delta z_j}  \left[ \frac{\partial G(z_i, z_j + \delta z_j)}{\partial z_i} \right]^T \frac{\partial G(z_i+\delta z_i, z_j)}{\partial z_j}  \\
			&  - \frac{1}{\delta z_i \delta z_j} \left[ \frac{\partial G(z_i, z_j + \delta z_j)}{\partial z_i} \right]^T  \frac{\partial G(z_i, z_j)}{\partial z_j}\\
			& - \frac{1}{\delta z_i \delta z_j} \left[ \frac{\partial G(z_i, z_j)}{\partial z_i} \right]^T  \frac{\partial G(z_i+\delta z_i, z_j)}{\partial z_j}\\
			& + \frac{1}{\delta z_i \delta z_j} \left[ \frac{\partial G(z_i, z_j)}{\partial z_i} \right]^T  \frac{\partial G(z_i, z_j)}{\partial z_j}. 
		\end{split}
		\label{eq:hessian4}
	\end{equation}
	where $ \frac{\partial G(z_i, z_j + \delta z_j)}{\partial z_i} $ is the partial gradient of $G$ at ($z_i$, $z_j + \delta z_j$) in the $z_i$ direction, and the other items are similarly defined.
	%
	%
	When the partial gradient is smooth with the small changes in $z_i$ and $z_j$, our OroJaR constrains both the last component and the other three components of Eqn.~(\ref{eq:hessian4}) to approach zero.
	Thus, OroJaR can offer an indirect and stronger regularization of Hessian Penalty.
	Moreover, OroJaR constrains the change caused by latent dimension in a holistic way, making it effective in disentangling latent dimensions corresponding to spatially correlated variations.
	
	\vspace{-1em}
	\section{Experiments}
	\vspace{-0.6em}
	
	\begin{figure*}
		\centering
		\includegraphics[width=0.99\linewidth]{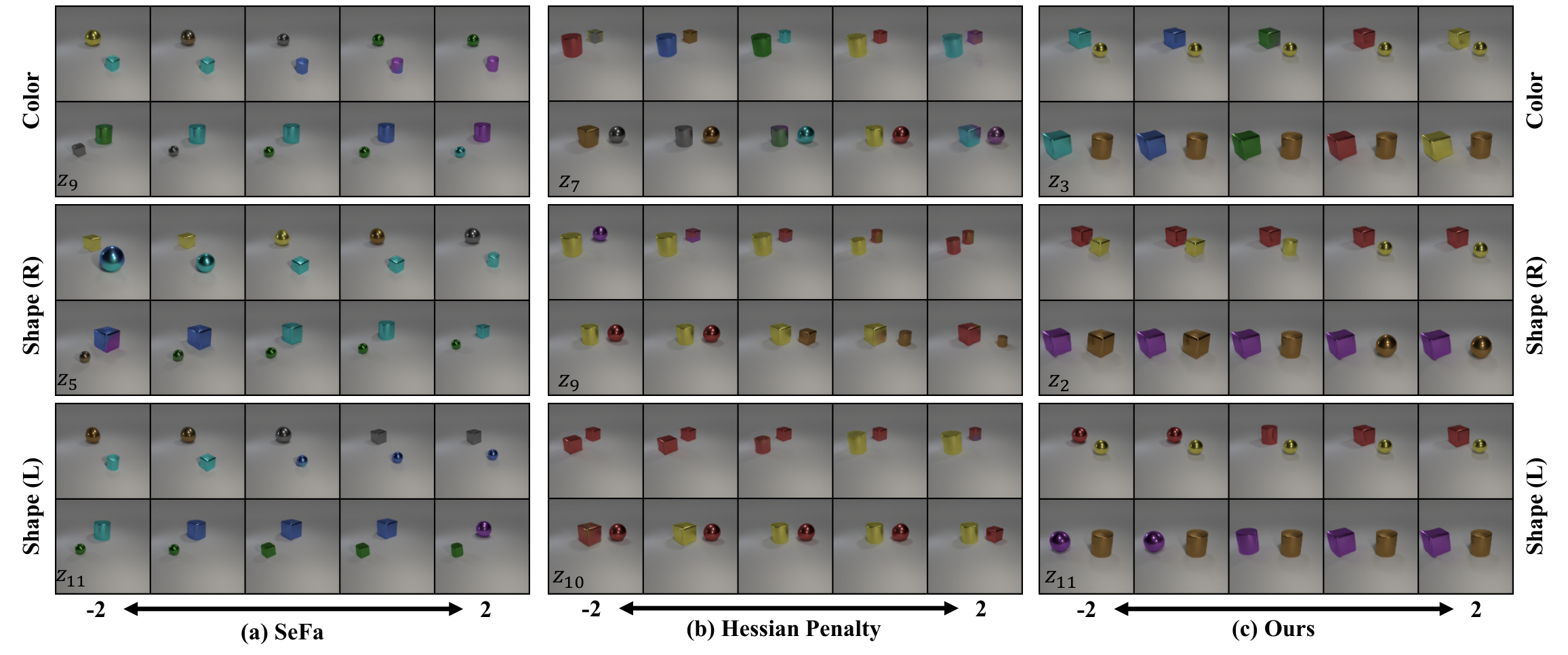}
		\vspace{-1.1em}
		\caption{Comparison of disentanglement quality by our OroJaR, Hessian Penalty \cite{peebles2020hessian} and SeFa \cite{shen2020closed} on CLEVR-Complex. 
			Here we show three representative factors discovered by all methods, \ie, color (Top),  shape of the rightmost object (Middle), and shape of the leftmost object (Bottom).
			\textbf{(a)} SeFa fails to disentangle the shape with color (see Middle and Bottom), and results in entangled representations.
			\textbf{(b)} Hessian Penalty performs poorly in controlling a single object while keeping another object unchanged.
			It learns to control the color of two objects by one dimension (see Top), and the shape or color of another object is also changed when changing the shape of one object (see Middle and Bottom).
			\textbf{(c)} Our OroJaR is effective in disentangling the color of leftmost object and the shape of each object.}
		\label{fig:visual_complex}
		\vspace{-1.6em}
	\end{figure*}
	
	In this section, we begin with an introduction of the datasets and implementation details, and then evaluate our OroJaR qualitatively and quantitatively by comparing it with the state-of-the-art methods. 
	A comprehensive ablation study is given in the \emph{suppl}.
	
	\vspace{-0.3em}
	\subsection{Datasets and Implementation Details}
	\vspace{-0.2em}
	
	\subsubsection{Datasets}
	\vspace{-0.2em}

	\textbf{Edges+Shoes.} Edges+Shoes \cite{yu2014fine} consists of 50,000 edges and 50,000 shoes images. 
	Following~\cite{peebles2020hessian}, we adopt this dataset to evaluate whether our method can discover an independent input component to control image-to-image translation without domain supervision.
	
	\noindent
	\textbf{CLEVR.} CLEVR dataset contians three synthetic datasets based on CLEVR \cite{johnson2017clevr}.
	The first dataset, CLEVR-1FOV, features a red cube with just a single factor of variation (FOV): object location along a single axis. 
	The second, CLEVR-Simple, has four FOVs: object color, shape, and location (both horizontal and vertical).
	The third, CLEVR-Complex, retains all FOVs from CLEVR-Simple and adds a second object and another FOV (\ie, object size), resulting in a total of ten FOVs (five per object).
	Each dataset consists of approximately 10,000 images.
	
	\noindent
	\textbf{Dsprites.} Dsprites \cite{dsprites17} contains totally 737,280 images generated from 5 independent latent factors (shape, size, rotation, horizontal and vertical positions).
	
	\vspace{-1em}
	\subsubsection{Implementation Details}
	\vspace{-0.5em}
	
	For Edges+Shoes and CLEVR datasets, we follow \cite{peebles2020hessian} to train the ProGAN \cite{karras2017progressive} on them and set the dimension of input to 12. 
	The image size is set to 128 $\times$ 128. 
	For the Dsprites dataset, we train a simple GAN (6 convolution layers), and the dimension of input is set to 6.
	The image size is set to 64 $ \times $ 64.
	In all the experiments, the OroJaR is applied right after the projection/convolution outputs for the first $D$ (10 for ProGAN and 4 for simple GAN) layers. We find that our OroJaR empirically achieves the best disentanglement performance when $ D $ corresponds to the last layer before the last upsampling layer.
	
	For BigGAN experiments, we set the $ N = m $ and restrict $ \mathbf{A} $ to be orthonormal by applying Gram-Schmidt and normalization during each forward pass.
	
	\begin{figure}
		\centering
		\includegraphics[width=1\linewidth]{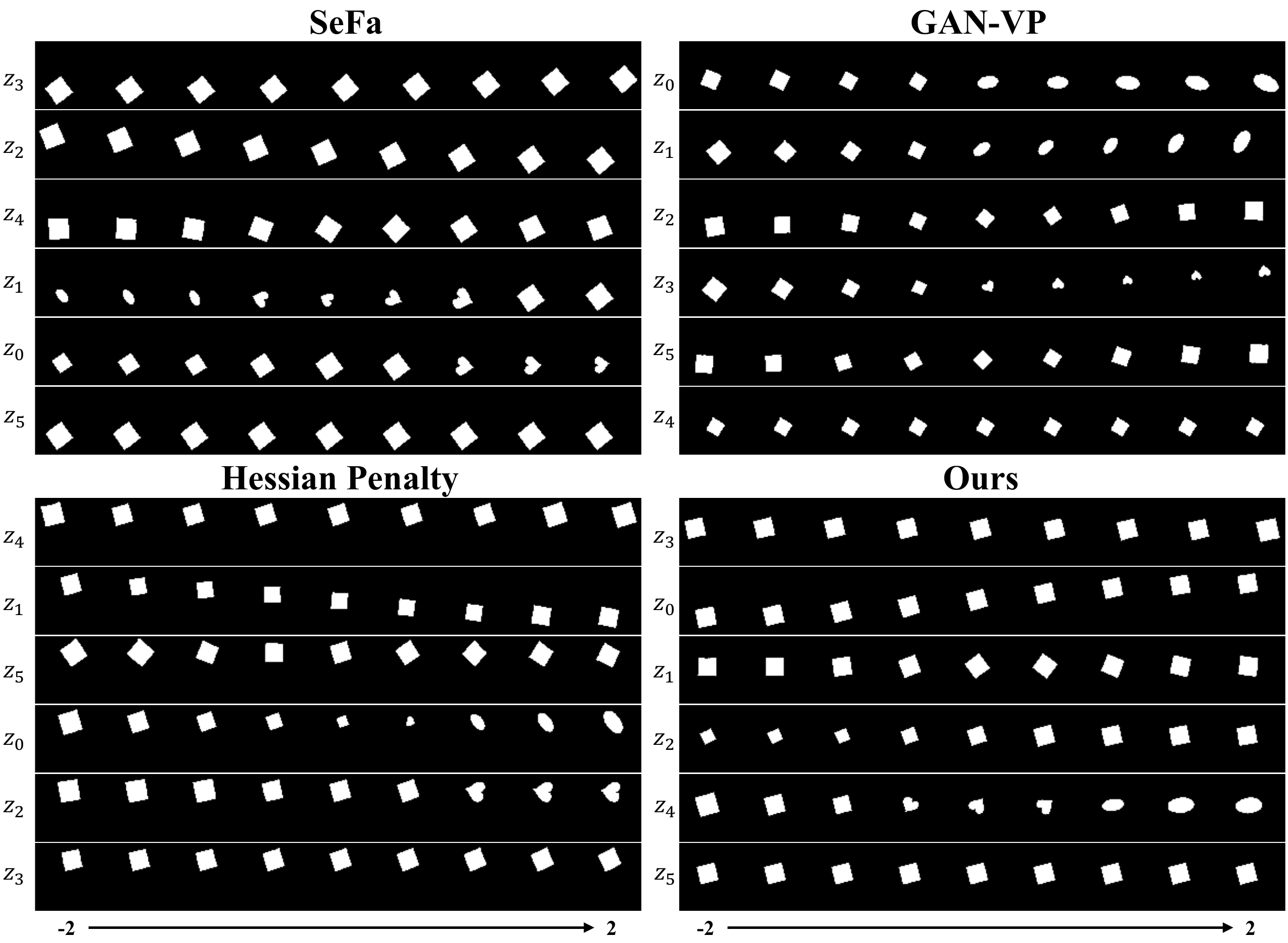}
		\vspace{-2em}
		\caption{Comparison of disentanglement quality by SeFa~\cite{shen2020closed}, GAN-VP~\cite{zhu2020learning}, Hessian Penalty~\cite{peebles2020hessian}, and our OroJaR on the Dsprites dataset. 
			\textbf{Top-Left:} SeFa~\cite{shen2020closed} entangles the rotation with the positions of object (2nd row). It also entangles the size factor with the shape factor (4th and 5th rows).
			\textbf{Top-Right:} For GAN-VP~\cite{zhu2020learning}, the positions are entangled with shape and rotation.
			\textbf{Bottom-Left:} Hessian Penalty~\cite{peebles2020hessian} entangles the rotation with positions, and also entangles the size with shape.
			\textbf{Bottom-Right:} Our method can successfully disentangle these five factors. From top to down, each row controls the horizontal position, vertical position, rotation, size, and shape, respectively. The latent dimension of the last row is correctly deactivated.
		}
		\label{fig:visual_dsprites}
		\vspace{-2em}
	\end{figure}
	\vspace{-0.5em}
	\subsection{Qualitative Evaluation}
	\vspace{-0.5em}
	In this subsection, we qualitatively compare the disentanglement quality of our OroJaR with three state-of-the-art disentanglement methods, \ie, SeFa~\cite{shen2020closed}, Hessian Penalty~\cite{peebles2020hessian}, and GAN-VP \cite{zhu2020learning}.
	
	\noindent\textbf{Edges+Shoes}.  
	Edges+Shoes dataset is a real-world but relatively simple dataset, where no ground-truth factors are provided. For a fair comparison, we choose the attributes corresponding to top two eigenvalues (lower value means ambiguous semantic direction) in SeFa.
	From Fig.~\ref{fig:visual_shoes}, SeFa, Hessian Penalty, and our OroJaR learn the same two major disentangled variations, \ie, the shape and style of shoes. While our method covers more diverse shapes.
	
	\noindent\textbf{CLEVR-Simple}.  Fig.~\ref{fig:visual_simple} shows the comparison on the CLEVR-Simple dataset.
	We note that the number of factors in this dataset is 4, while the dimension of input is 12. 
	When the latent space is overparameterized, our OroJaR can automatically turn off the extra dimensions. 
	Here we only compare the top six activeness scoring dimensions with the competing methods (See Fig.~\ref{fig:activeness} and Sec.~\ref{sec:quantitative}). 
	The remaining dimensions are deactivated based on both our OroJaR and Hessian Penalty~\cite{peebles2020hessian}, and thus are not shown.
	From Fig.~\ref{fig:visual_simple}, SeFa learns to control the horizontal and vertical positions of the object (top two rows), but entangles the color with the shape variations (3rd-5th rows).
	Hessian Penalty successfully disentangles the vertical position, shape, and color of the object (2nd-5th rows), but the horizontal position is unexpectedly controlled by two dimensions (1st and 6th rows). 
	In comparison, our method successfully disentangles the four factors (top five rows) and deactivates the extra dimension (6th row).

	\noindent\textbf{CLEVR-Complex}.  Fig.~\ref{fig:visual_complex} shows the comparison on the CLEVR-Complex dataset. 
	Obviously, SeFa fails to disentangle the shape with color variations.
	Hessian Penalty performs poorly in controlling a single object while keeping another object unchanged.
	When changing the shape of one object, the shape or color of another object is also changed at the same time.
	A possible explanation is that Hessian Penalty constrains each entry of output independently.
	This makes it not sufficient in disentangling the complex latent directions (\eg, shape and color of an object) corresponding to spatially correlated variations.
	On the contrary, our OroJaR effectively disentangles the color of the leftmost object and the shape of each object, and thus learns a better disentangled representation.
	
	\noindent\textbf{Dsprites}.  Fig.~\ref{fig:visual_dsprites} shows the qualitative comparison with SeFa~\cite{shen2020closed}, Hessian Penalty~\cite{peebles2020hessian}, and GAN-VP \cite{zhu2020learning} on the Dsprites dataset.
	GAN-VP~\cite{zhu2020learning} is still limited in learning disentangled representations, where the positions are entangled with the shape and rotation.
	As for Hessian Penalty~\cite{peebles2020hessian} and SeFa~\cite{shen2020closed}, the positions of object are entangled with the rotation.
	They also fail to disentangle the shape with the size variation.
	In contrast, our OroJaR can successfully disentangle these five factors while correctly deactivating the latent dimension of the last row.
	The results indicate that our OroJaR is superior in disentangling spatially correlated variations (\eg, shape, size, rotation, \etc).
	
	\noindent\textbf{BigGAN}. 
	According to Sec.~\ref{sec: app}, our OroJaR can also be used to discover the meaningful latent directions of pre-trained GAN. Here we apply it to class-conditional BigGAN \cite{brock2018large} trained on ImageNet \cite{deng2009imagenet}. Fig.~\ref{fig:biggan} shows our results on Golden Retrievers and Churches, and our method is able to discover several disentangled directions, such as rotate, zoom, and color. Fig.~\ref{fig:biggan1} shows the qualitative comparison with Hessian Penalty \cite{peebles2020hessian} and Voynov~\cite{voynov2020unsupervised}. Voynov~\cite{voynov2020unsupervised} entangles the color of the dog with zoom variation. Hessian Penalty entangles the rotation with zoom variation. In contrast, our OroJaR performs a better zoom quality.
	
	\noindent\textbf{More Results.} More qualitative results (\eg CLEVR-U, CLEVR-1FOV, and BigGAN) are given in the suppl.
	
	\begin{figure}
		\centering
		\includegraphics[width=1\linewidth]{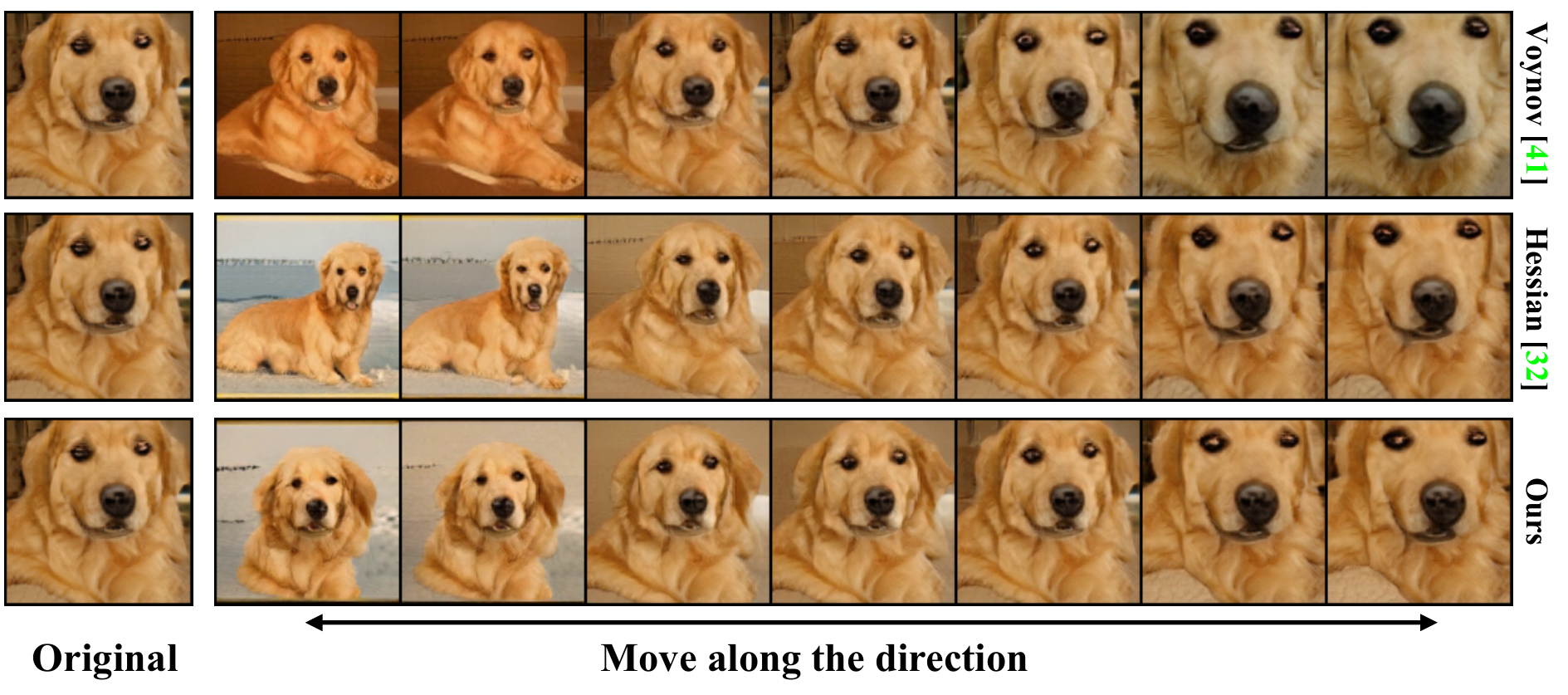}
		\vspace{-2em}
		\caption{
			Comparing the quality of latent space editing by our OroJaR, Hessian Penalty \cite{peebles2020hessian}, and Voynov \cite{voynov2020unsupervised}. The direction is added from $\eta$ = -2.5 to 2.5 for Hessian Penalty and our OroJaR, and from -8 to 8 for Voynov.  Our OroJaR better disentangles zoom from rotation and color.
		}
		\label{fig:biggan1}
		\vspace{-1.4em}
	\end{figure} 
	
	\vspace{-0.5em}
	\subsection{Quantitative Evaluation \label{sec:quantitative}}
	\vspace{-0.5em}
	
	\begin{table*}[t]
		{\small
			\centering
			\caption{Comparison of Perceptual Path Length (PPL), Frchet Inception Distance (FID) and Variation Predictability Metric (VP) for different methods on Edges+Shoes and CLEVR. For FID and PPL, lower is better, and for VP, higher is better.  We report the model with the best FID within the same number of training iterations. PPL, FID, and VP are computed with 100,000, 50,000 and 10,000 samples. The CLEVR-U dataset indicates that we train the model on CLEVR-Simple by setting $ m = 3 $. Due to CLEVR-1FOV only has one factor, we do not report the VP results on it.} %
			\label{tab:ppls}
			\vspace{-0.8em}
			\resizebox{1.0\linewidth}{!}{
				\aboverulesep=0ex
				\belowrulesep=0ex
				\renewcommand{\arraystretch}{1.2}
				\begin{tabular}{lccccccccccccccc}
					\toprule
					\multirow{2}[3]{*}{Method} &
					\multicolumn{3}{c}{Edges$+$Shoes} &
					\multicolumn{3}{c}{CLEVR-Simple} &
					\multicolumn{3}{c}{CLEVR-Complex} &
					\multicolumn{3}{c}{CLEVR-U} &
					\multicolumn{3}{c}{CLEVR-1FOV} \\
					\cmidrule(llr){2-4}
					\cmidrule(llr){5-7}
					\cmidrule(llr){8-10}
					\cmidrule(llr){11-13}
					\cmidrule(llr){14-16}
					& PPL ($ \downarrow $) & FID  ($ \downarrow $) & VP ($ \uparrow $) & PPL & FID & VP & PPL & FID & VP & PPL & FID & VP & PPL & FID & VP\\
					\midrule
					
					InfoGAN\cite{chen2016infogan}  & 2952.2 & \textbf{10.4} & 15.6 & 56.2 & \textbf{2.9} & 28.7 & 83.9 & \textbf{4.2} & 27.9 & 766.7 & 3.6 & 40.1 & 22.1 & 6.2 & - \\
					
					ProGAN \cite{karras2017progressive}  & 3154.1 & 10.8 & 15.5 & 64.5 & 3.8 & 27.2 & 84.4 & 5.5 & 25.5 & 697.7 & \textbf{3.4} & 40.2 & 30.3 & 9.0 & - \\
					
					SeFa\cite{shen2020closed}  & 3154.1 & 10.8 & 24.1 & 64.5 & 3.8 & 58.4 & 84.4 & 5.5 & 30.9 & 697.7 & \textbf{3.4} & 42.0 & 30.3 & 9.0 & - \\
					
					Hessian Penalty \cite{peebles2020hessian} & 554.1 & 17.3 & 28.6 & 39.7 & 6.1 & 71.3 & 74.7 & 7.1 & 42.9 & 61.6 & 26.8 & 79.2 & 20.8 & 2.3 & - \\
					Ours & \textbf{236.7} & 16.1 & \textbf{32.3} & \textbf{6.7} & 4.9 & \textbf{76.9} & \textbf{10.4} & 10.7 & \textbf{48.8} & \textbf{40.9} & 4.6 & \textbf{90.7} & \textbf{2.8} & \textbf{2.1} & - \\
					\bottomrule
					\vspace{0.1pt}
				\end{tabular}
			}
		}
		\vspace{-2.6em}
	\end{table*}
	
	\begin{table}[t]
		\begin{center}
			\caption{Comparison of Variation Predictability Metric (VP) for different methods on Dsprites.}
			\vspace{-0.8em}
			\label{tab:vp}
			\resizebox{1.0\linewidth}{!}{
				\aboverulesep=0ex
				\belowrulesep=0ex
				\renewcommand{\arraystretch}{1.2}
				\begin{tabular}{lccccc}
					\hline\noalign{\smallskip}
					Method & GAN & SeFa & GAN-VP & Hessian Penalty & Ours  \\
					\hline
					VP(\%, $ \uparrow $) & 30.9 (0.84) & 48.6 (0.70) & 39.1 (0.48) & 48.5 (0.56) & \textbf{54.7 (0.27)} \\
					\hline
				\end{tabular}
			}
		\end{center}
		\vspace{-1.9em}
	\end{table}
	
	In this subsection, we quantitatively compare our OroJaR with several state-of-the-art deep generative models. 
	Following~\cite{peebles2020hessian}, we use Perceptual Path Length (PPL) and Frechet Inception Distance (FID) as the quantitative metrics. 
	PPL~\cite{karras2019style} measures the smoothness of the generator by evaluating how much $ G(\mathbf{z}) $ changes under perturbations to $ \mathbf{z} $.
	While FID~\cite{heusel2017gans} exploits the distance between activation distributions for measuring the quality of generated images.
	However, neither PPL nor FID are designed for assessing disentanglement performance. 
	So we also report the Variation Predictability Disentanglement Metric (VP)~\cite{zhu2020learning} in the quantitative evaluation.

	Table~\ref{tab:ppls} lists the quantitative comparison results on the Edges+Shoes and the CLEVR datasets.
	The CLEVR-1FOV dataset has only one factor and all the competing methods have the same VP value. 
	So we do not report the VP results on this dataset. 
	From Table~\ref{tab:ppls}, our OroJaR achieves better VP results on all datasets, indicating that it can learn better disentangled representation. 
	Besides, it also serves as a path length regularization in \cite{karras2020analyzing} and helps learn a smooth latent space, resulting better PPL results.
	For our OroJaR, we empirically find that removing the normalization and activation of the first fully-connected layer is beneficial to the improvements on disentanglement. 
	Albeit InfoGAN~\cite{chen2016infogan} gets lower FID on most datasets, it performs poorly in learning disentangled representation.
	Table~\ref{tab:vp} lists the VP results on the Dsprites dataset, and our OroJaR also achieves the highest VP among the competing methods, indicating that our OroJaR performs favorably against the state-of-the-art methods for unsupervised disentanglement learning.

	\begin{figure}
		\centering
		\includegraphics[width=1\linewidth]{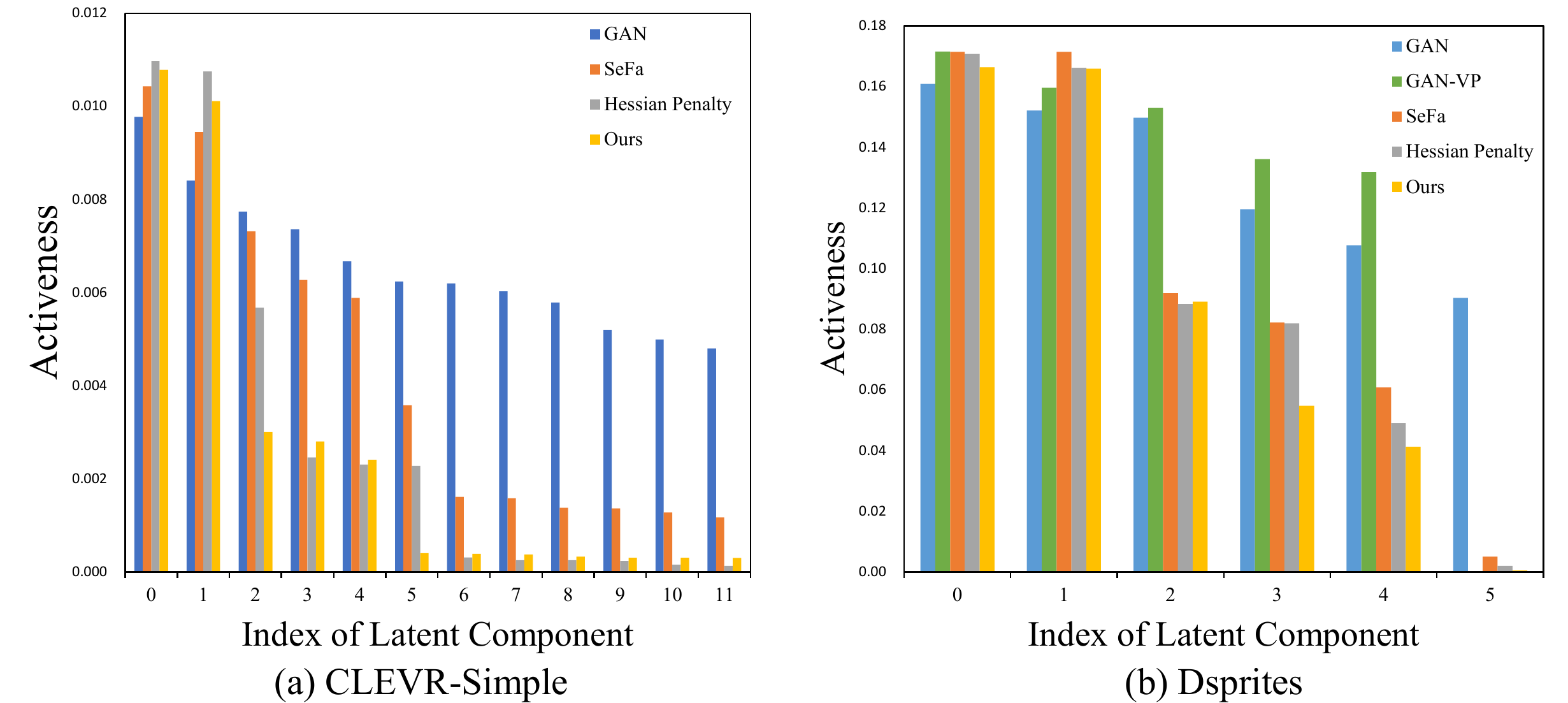}
		\vspace{-1.9em}
		\caption{Comparison of Activeness Scores (how much each dimension controls $G$'s output) on CLEVR-Simple and Dsprites.
			(a) On CLEVR-Simple, both our OroJaR and Hessian Penalty \cite{peebles2020hessian} can deactivate the redundant dimensions (5/6 of 12 are activated). (b) On Dsprites, we have similar observation. SeFa \cite{shen2020closed} and GAN-VP \cite{zhu2020learning} also have the ability to deactivate redundant dimensions.}
		\label{fig:activeness}
		\vspace{-1.9em}
	\end{figure}

	In many practical scenarios, we do not have sufficient prior to setting the number of disentangled factors.
	One feasible solution is to use a larger dimension of input, and the disentanglement algorithm is able to identify and turn off redundant dimensions. 
	Following~\cite{peebles2020hessian}, the activeness of a dimension $ z_i $ is introduced as the mean variance of $ G(\mathbf{z}) $ as we change $ z_i $ while keeping the other dimensions fixed.
	For assessing the ability to find redundant dimensions, Fig.~\ref{fig:activeness} shows the activeness scores on CLEVR-Simple and Dsprites. 
	In comparison to the GAN counterpart, both SeFa~\cite{shen2020closed}, Hessian Penalty~\cite{peebles2020hessian}, and our OroJaR is able to find redundant dimensions with smaller activeness scores.
	However, SeFa~\cite{shen2020closed} and Hessian Penalty~\cite{peebles2020hessian} fails to find all the redundant dimensions, which can also be observed from Fig.~\ref{fig:visual_simple}. 
	As for GAN-VP~\cite{zhu2020learning}, we note that the VP loss encourages the variation caused by each dimension of $ \mathbf{z} $ to be distinguishable.
	Consequently, it can only deactivate at most one dimension, and the dimension of $\mathbf{z}$ should be carefully set to ensure GAN-VP works well. So we do not report the results of GAN-VP on Edges+Shoes and CLEVR, in which the dimension of input is set to 12 and is higher than the number of FOVs.
	
	\vspace{-0.5em}
	\section{Conclusion}
	\vspace{-0.5em}
	
	In this paper, we proposed an Orthogonal Jacobian Regularization (OroJaR) to help the generative model in learning disentangled representations. 
	It encourages disentanglement by constraining the changes of output caused by different latent dimensions (\ie, Jacobian vectors) to be orthogonal. 
	Moreover, our OroJaR can be applied to multiple layers of the generator, and constrains the output in a holistic way, making it effective in disentangling latent dimensions corresponding to spatially correlated variations.
	Experimental results demonstrate that our OroJaR is effective in disentangled and controllable image generation, and performs favorably against the state-of-the-art methods.  In the future, we will extend OroJaR to VAE and other generative models for improving disentanglement learning.
	
	\vspace{-0.5em}
	\section*{Acknowledgement}
	\vspace{-0.5em}
	
	This work was supported in part by National Key R\&D Program of China under Grant No. 2020AAA0104500, and by the National Natural Science Foundation of China (NSFC) under Grant No.s U19A2073 and 62006064.

	\newpage
	{\small
		\bibliographystyle{ieee_fullname}
		\bibliography{egbib}
	}
	
	\newpage
	
	\begin{table*}
		\begin{center}
			\caption{Comparison of Variation Predictability Metric (VP) for different settings and SeFa on Dsprites.}
			\label{tab:ablation}
			\resizebox{0.8\linewidth}{!}{
				\aboverulesep=0ex
				\belowrulesep=0ex
				\renewcommand{\arraystretch}{1.2}
				\begin{tabular}{lcccccccccccc}
					\hline\noalign{\smallskip}
					Method & GAN & L0 & L1 & L2 & L3 & L4 & L0$\sim$1 & L0$\sim$2 & L0$\sim$3 & L0$\sim$4 & SeFa & Ours(L0$\sim$3)  \\
					\hline
					VP(\%, $ \uparrow $) & 30.9 & 47.6 & 48.1 & 50.2 & 36.4 & 35.0 & 48.8 & 53.5 & \textbf{54.7} & 52.3 & 48.6 & \textbf{54.7} \\
					\hline
				\end{tabular}
			}
		\end{center}
	\end{table*}
	
	\section*{A. Proof of Proposition}
	
	We first give a brief proof of the proposition that related to SeFa \cite{shen2020closed} in the main paper.
	
	\newtheorem{prop}{Proposition}
	\begin{prop}
		Let $\mathbf{W} = \mathbf{U} \boldsymbol{\Lambda} \mathbf{V}^T$ be the singular value decomposition (SVD) of the weight parameter $\mathbf{W}$. Let $\mathbf{z}^{\prime} = \mathbf{V}^T \mathbf{z}$, $\mathbf{W}^{\prime} = \mathbf{U} \boldsymbol{\Lambda}$, and define $G_1(\mathbf{z}) = \mathbf{W}\mathbf{z}$, $G^{\prime}_1(\mathbf{z}^{\prime}) = \mathbf{W}^{\prime}\mathbf{z}^{\prime}$. We have, 
		\begin{enumerate}
			\item $G^{\prime}_1(\mathbf{z}^{\prime})$ is equivalent with $G_1(\mathbf{z})$, i.e., $G_1(\mathbf{z}) = G^{\prime}_1(\mathbf{z}^{\prime})$.
			\item Hard orthogonal Jacobian constraint can be attained, i.e.,
			\begin{equation}
				\left[ \frac{\partial G^{\prime}_1}{\partial z^{\prime}_i}\right]^T \frac{\partial G^{\prime}_1}{\partial z^{\prime}_j} = 0.
			\end{equation}
		\end{enumerate}
	\end{prop}

	\begin{proof} \quad
		\newline 1. 
		\begin{equation}
			\begin{split}
				G_1(\mathbf{z}) &= \mathbf{W}\mathbf{z} = \mathbf{U} \boldsymbol{\Lambda} \mathbf{V}^T\mathbf{z} = \mathbf{U} \boldsymbol{\Lambda} \mathbf{z}^{\prime} \\
				&= \mathbf{W}^{\prime}\mathbf{z}^{\prime} = G^{\prime}_1(\mathbf{z}^{\prime}) .
			\end{split}
		\end{equation}
		2. 
		\begin{equation}
			\label{eqn:jacobian}
			\begin{split}
				\left[ \frac{\partial G^{\prime}_1}{\partial \mathbf{z^{\prime}}}\right]^T \frac{\partial G^{\prime}_1}{\partial \mathbf{z^{\prime}}} & = \left[ \mathbf{W}^{\prime} \right]^T \mathbf{W}^{\prime} =  \boldsymbol{\Lambda}^T \mathbf{U}^T  \mathbf{U} \boldsymbol{\Lambda} \\
				&= \boldsymbol{\Lambda}^2,
			\end{split}
		\end{equation}
		where $\boldsymbol{\Lambda}$ is diagonal. $\left[ \frac{\partial G^{\prime}_1}{\partial z^{\prime}_i}\right]^T \frac{\partial G^{\prime}_1}{\partial z^{\prime}_j}$ is the off-diagonal entry of Eqn.~\ref{eqn:jacobian} when $ i \neq j $, and thus to be zero.
	\end{proof}
	
	\section*{B. Additional Qualitative Results}
	
	\subsection*{B.1 CLEVR-1FOV Dataset}

	\begin{figure*}
		\centering
		\includegraphics[width=0.99\linewidth]{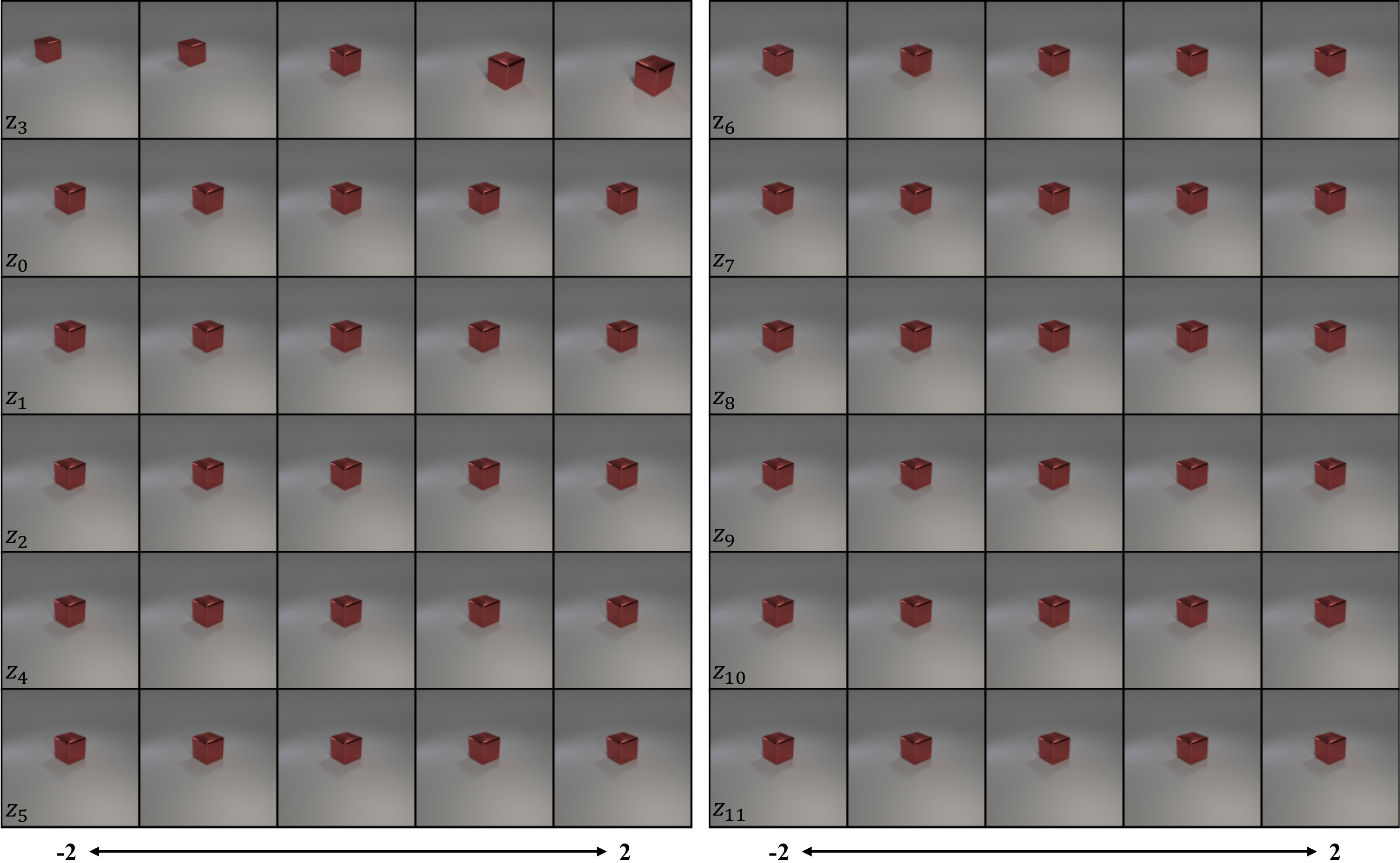}
		\caption{Qualitative results on the CLEVR-1FOV by our OroJaR. We randomly sample a 12-dimensional input vector $\mathbf{z}$ from a normal distribution, and each row corresponds to one of the dimensions. Moving across a row, we vary the value of dimension $ z_i $ from $-2$ to $+2$ while keeping the other 11 dimensions unchanged. It can be seen that, the redundant dimensions are successfully deactivated. While the only activated dimension (the top left row) controls the unique factor of variation in the dataset.}
		\label{fig:1fov}
	\end{figure*}
	
	Fig.~\ref{fig:1fov} shows the qualitative results on the CLEVR-1FOV by our OroJaR. CLEVR-1FOV has only one factor of variation: a red cube's location along a single axis. From Fig.~\ref{fig:1fov}, our OroJaR can successfully deactivate the redundant dimensions while controlling the position of the object with the only activated dimension (the top left row).
	
	\subsection*{B.2 CLEVR-U Dataset}

	\begin{figure*}
		\centering
		\includegraphics[width=0.99\linewidth]{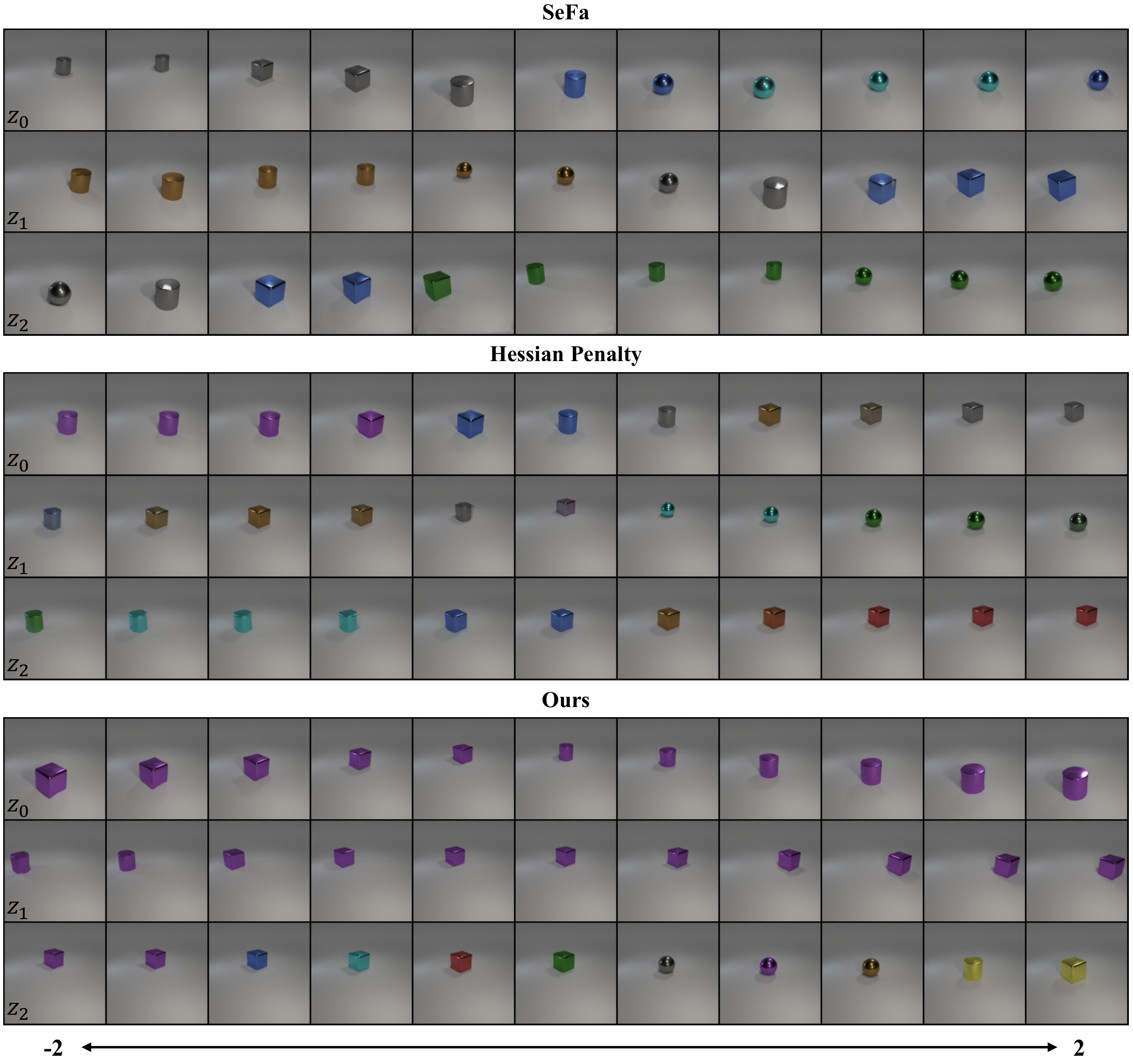}
		\caption{Comparison of disentanglement quality by SeFa~\cite{shen2020closed}, Hessian Penalty~\cite{peebles2020hessian}, and our OroJaR on the CLEVR-U dataset. CLEVR-U indicates that we trained the model on CLEVR-Simple (4 factors) by setting the dimension of input to 3. For each method, we randomly sample a 3-dimensional Gaussian vector, and each row corresponds to one of the dimensions.  \textbf{Top}: In this underparameterized setting, SeFa learns an entangled representation, and each dimension controls all four variations at the same time. \textbf{Middle}: Hessian Penalty also entangles the shape with position variation. When moving the object along a direction, the shape of the object is also changed at the same time (the last two rows). \textbf{Bottom}: Our method learns a better disentangled result. From top to down, each dimension controls the variations of vertical position, horizontal position, and shape (entangles with color), respectively.}
		\label{fig:clevru}
	\end{figure*}
	
	Fig.~\ref{fig:clevru} shows the qualitative comparison with SeFa \cite{shen2020closed} and Hessian Penalty \cite{peebles2020hessian} on the CLEVR-U dataset. CLEVR-U indicates that we train the model on CLEVR-Simple (4 factors of variation, \ie, horizontal and vertical positions, shape, and color) by setting the dimension of input to 3, which is an underparameterized setting. Obviously, SeFa fails to disentangle the four factors, which is shown that each dimension controls all four variations at the same time. Hessian Penalty also entangles the position with shape variation (\eg, 2nd and 3rd rows). On the contrary, our OroJaR still learned to control the variations of horizontal position, vertical position, and shape (entangles with color) independently. 
	The results indicate that our OroJaR is superior in disentangling spatially correlated variations (\eg, shape and position).

	\subsection*{B.3 BigGAN}
	
	Fig.~\ref{fig:biggan_full} shows a more comprehensive qualitative comparison with Hessian Penalty \cite{peebles2020hessian} and Voynov~\cite{voynov2020unsupervised} on BigGAN. One can see that, three methods discover similar latent directions (\ie, zoom, rotate, and smoosh nose). However, Voynov~\cite{voynov2020unsupervised} gives a degraded results when rotating the dog to the left side. It is worth noting that, Hessian Penalty finds directions that are very similar to ours. For rotation and smoosh nose directions, the cosine similarities between our method and Hessian Penalty reach 0.99. This may be caused by the limitation of pre-trained GAN. Nonetheless, compared with Hessian Penalty, our method performs a better zoom editing quality.
	
	\begin{figure*}
		\centering
		\includegraphics[width=0.99\linewidth]{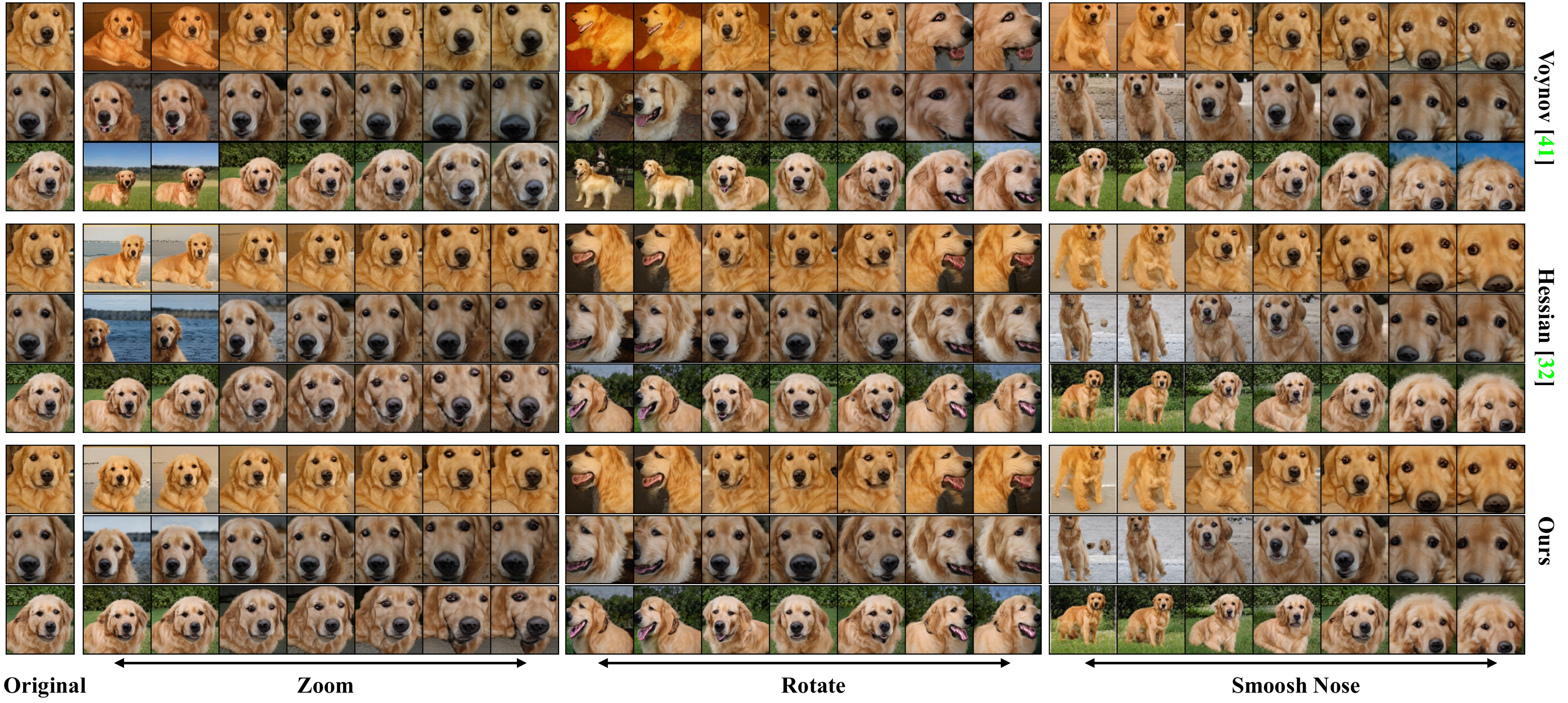}
		\vspace{-0.5em}
		\caption{
			Comparing the quality of latent space editing by our OroJaR, Hessian Penalty \cite{peebles2020hessian} and Voynov \cite{voynov2020unsupervised}. Voynov \cite{voynov2020unsupervised} learns entangled rotation factor, and gets a degraded results when rotating the dog to the left side. Hessian Penalty learns a similar rotation and smoosh nose factors with ours, but our method achieves a better zoom editing quality.
		}
		\label{fig:biggan_full}
	\end{figure*}

	\section*{C. Ablation Study}
	
	\begin{figure*}
		\centering
		\includegraphics[width=0.99\linewidth]{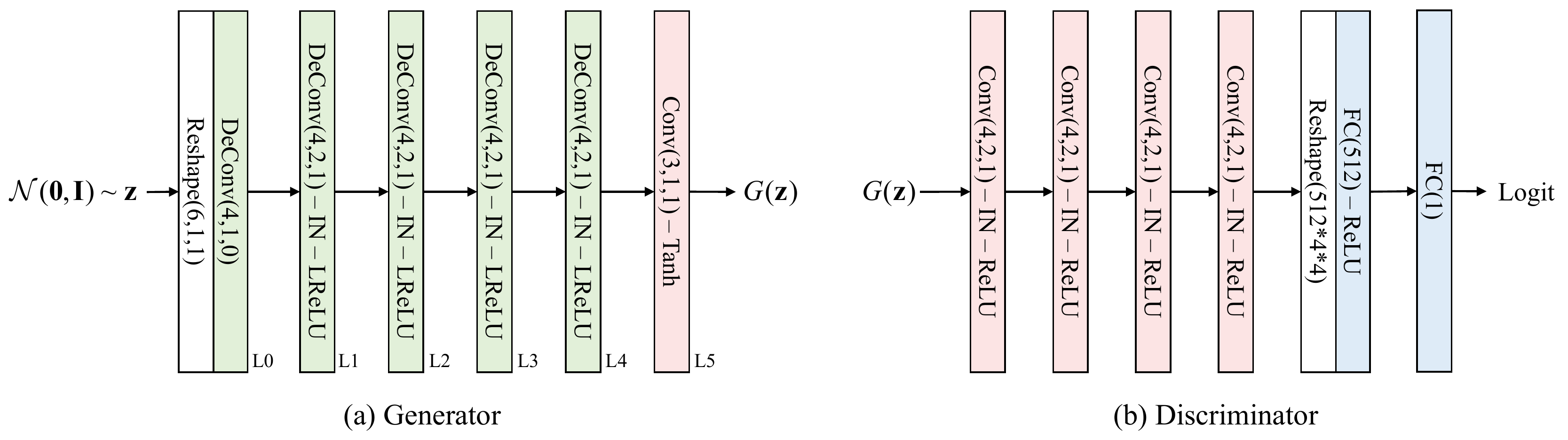}
		\caption{
			Network architectures of simple GAN used on Dsprites experiments. Conv(\textit{k}, \textit{s}, \textit{p}) and DeConv(\textit{k}, \textit{s}, \textit{p}) denote convolutional layer and transposed convolutional layer where \textit{k} is kernel size, \textit{s} is stride and \textit{p} is padding size. FC(\textit{d}) denotes fully connected layer with \textit{d} as output dimension. LReLU denotes the Leaky ReLU nonlinearity.
		}
		\label{fig:simplegan}
	\end{figure*}

	To demonstrate the effectiveness of our OroJaR, we train a simple GAN (6 layers, the network architectures are shown in Fig.~\ref{fig:simplegan}.) on the Dsprites dataset under three different settings:
	\begin{itemize}
		\setlength{\itemsep}{0pt}
		\setlength{\parsep}{0pt}
		\setlength{\parskip}{0pt}
		\item Firstly, we train the GAN without applying the OroJaR (GAN). 
		\item Secondly, we train the GAN by applying OroJaR to every single intermediate layer (L0 to L4). 
		\item Thirdly, we train the GAN by applying the OroJaR to the first multiple layers (L0$ \sim $2 to L0$ \sim $4). 
	\end{itemize}
	Fig.~\ref{fig:ablation} and Table~\ref{tab:ablation} show the qualitative and quantitative comparison among these settings.
	The results on a single intermediate layer (L0 to L4) show that the earlier layers are more effective in deactivating the redundant dimensions, resulting in better disentanglement. 
	By applying OroJaR to the deeper layer, the benefit of OroJaR to the disentanglement first increases and then sharply decreases. 
	Nonetheless, applying OroJaR to a single intermediate layer is not sufficient to learn a well-disentangled model, and applying OroJaR to the first multiple layers helps learn a better disentangled model.  
	Our OroJaR empirically achieves the best disentanglement performance when $ D $ corresponding to the last layer before the last upsampling layer.

	\begin{figure*}
		\centering
		\includegraphics[width=0.99\linewidth]{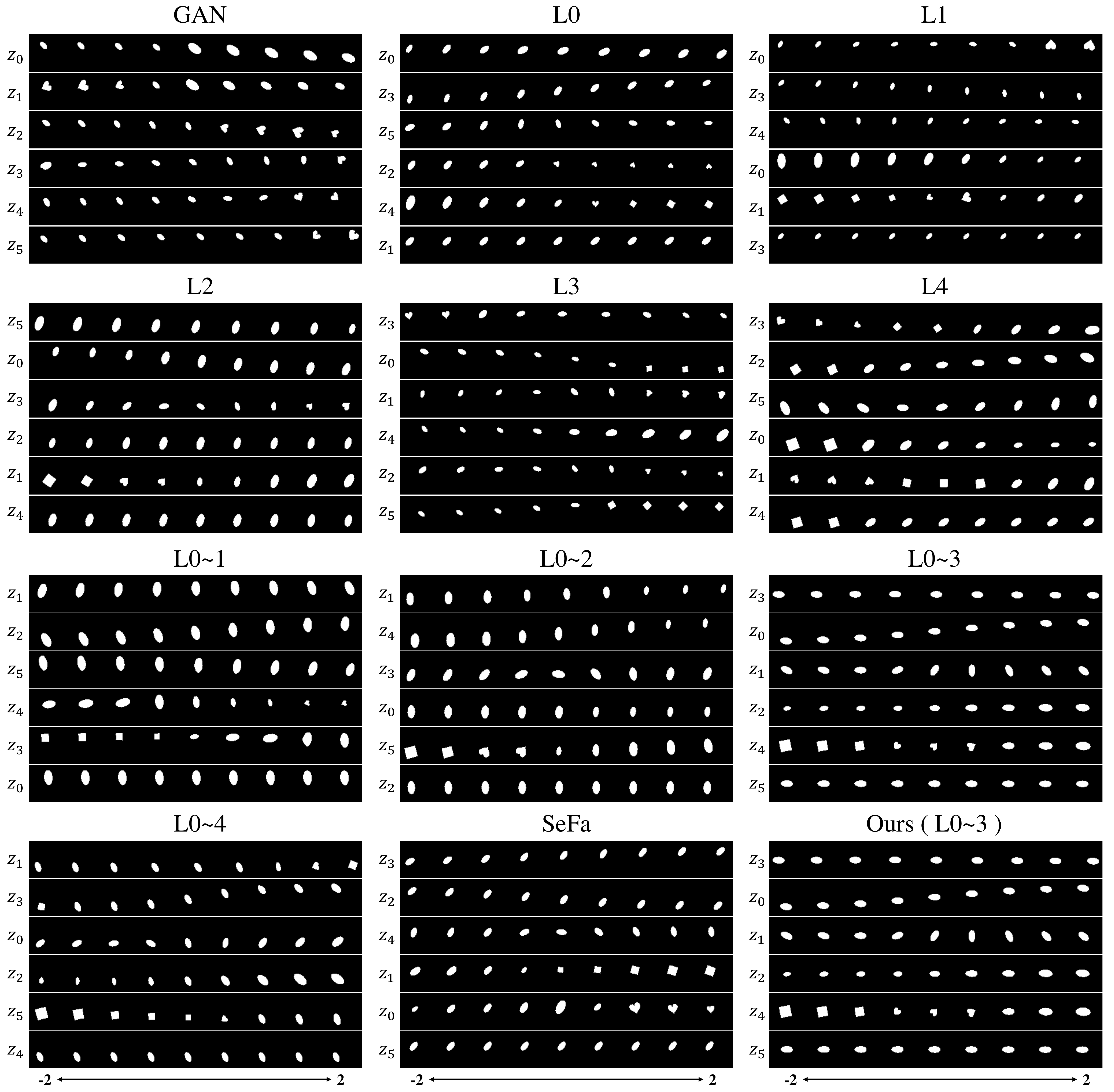}
		\vspace{-0.5em}
		\caption{
			Effectiveness of our OroJaR in disentanglement learning. L$ x $ means OroJaR is applied to the $ x $-th layer, and L$0 \sim x $ means OroJaR is applied to first $ x $ layers. We found that earlier layers are more effective in deactivating the redundant dimensions, resulting in better disentanglement. Applying OroJaR to a single intermediate layer is not sufficient to learn a well-disentangled model, and applying OroJaR to first multiple layers helps learn a better disentangled model.
		}
		\label{fig:ablation}
	\end{figure*}

\end{document}